\documentclass[final]{cvpr}

\usepackage{times}
\usepackage{epsfig}
\usepackage{graphicx}
\usepackage{amsmath}
\usepackage{amssymb}
\usepackage{enumerate}
\usepackage{amsfonts}       
\usepackage{amsthm}
\usepackage{color}
\usepackage{algorithm,algpseudocode}
\usepackage{multicol}
\newtheorem{lemma}{Lemma}

\usepackage{graphicx}
\usepackage{subfigure}
\usepackage{wrapfig}

\usepackage{hyperref}
\hypersetup{pagebackref=true,breaklinks=true,colorlinks,bookmarks=false}

\def\cvprPaperID{7483} 
\def\confYear{CVPR 2021}

\def\ceil#1{\lceil #1 \rceil}
\def\round#1{\lfloor #1 \rceil}
\begin{document}

\title{Continual Learning via Bit-Level Information Preserving}

\author{Yujun Shi\textsuperscript{\rm 1}\thanks{Code: \href{https://github.com/Yujun-Shi/BLIP}{https://github.com/Yujun-Shi/BLIP}}
\quad
Li Yuan\textsuperscript{\rm 1}\thanks{Corresponding author.}
\quad
Yunpeng Chen\textsuperscript{\rm 2}
\quad
Jiashi Feng\textsuperscript{\rm 1}\\

\textsuperscript{\rm 1}National University of Singapore \quad \textsuperscript{\rm 2} YITU Technology\\
{\tt\small \{shi.yujun,yuanli\}@u.nus.edu}
\quad
{\tt\small yunpeng.chen@yitu-inc.com}
\quad
{\tt\small elefjia@nus.edu.sg}
}

\maketitle
\thispagestyle{empty}
\begin{abstract}

Continual learning tackles the setting of learning different tasks sequentially. Despite the lots of previous solutions, most of them still suffer significant forgetting or expensive memory cost. In this work, targeted at these problems, we first study the continual learning process through the lens of information theory and observe that forgetting of a model stems from the loss of \emph{information gain} on its parameters from the previous tasks when learning a new task. From this viewpoint, we then propose a novel continual learning approach called Bit-Level Information Preserving (BLIP) that preserves the information gain on model parameters through updating the parameters at the bit level, which can be conveniently implemented with parameter quantization. More specifically, BLIP first trains a neural network with weight quantization on the new incoming task and then estimates information gain on each parameter provided by the task data to determine the bits to be frozen to prevent forgetting. We conduct extensive experiments ranging from classification tasks to reinforcement learning tasks, and the results show that our method produces better or on par results comparing to previous state-of-the-arts. Indeed, BLIP achieves close to zero forgetting while only requiring constant memory overheads throughout continual learning.
\end{abstract}

\section{Introduction}
Continual learning tackles the setting where an agent learns different tasks sequentially.
It is challenging since the agent is usually not allowed to refer to the previously learned tasks when learning a new one.
Current artificial neural networks generally fail as they tend to suffer severe performance degradation on previously learned tasks after learning new ones, which is known as \emph{catastrophic forgetting} \cite{mccloskey1989catastrophic,french1999catastrophic}.
A commonly acknowledged reason for this problem is that model parameters drift when fitting the new incoming task data.

\begin{figure*}
  \centering
  \includegraphics[width=16.5cm]{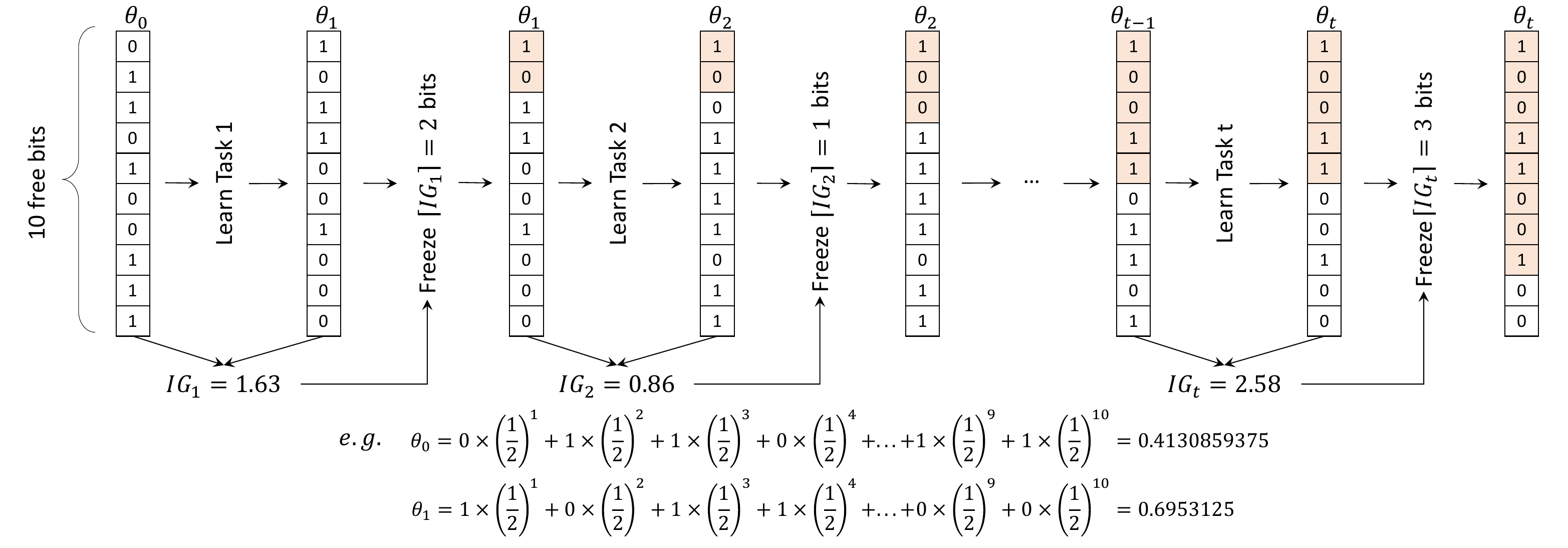}
  \caption{\textbf{Workflow of BLIP. Best viewed in color.}
  We consider a simple scenario with one single parameter $\theta$ quantized to 10 bits to illustrate our method.
  $\theta_{t}$ denotes the parameter after learning on task $1$ to $t$,
  and $\theta_{0}$ is a randomly initialized value before training on any task.
  $IG_{t}$ denotes information gain on $\theta$ after learning the task $t$.
  Bit representation of $\theta$ after learning each task is shown below.
  From the higher bit positions to lower ones is more significant bits to less significant ones.
  Frozen bits are filled with color and the rest bits are free bits.
  After learning each task, the information gain is calculated and then $\ceil{IG_{t}}$ bits are to be frozen in the bit representation.
 By repeating this process, the information on previous tasks can be preserved, enabling continual learning for neural networks.
 }
  \label{figure1}
  \vspace{-2mm}
\end{figure*}

Two types of methods have been developed to address this hazard.
The first type of methods, which are based on model pruning technique~\cite{fernando2017pathnet,serra2018overcoming,mallya2018packnet,mallya2018piggyback}, first identify model parameters important for a task and then store a task-specific mask to mark these parameters and prohibit them from changing in subsequent learning to prevent forgetting.
During the evaluation stage, only the parameters marked by the task-specific mask are applied on the given task.
While these methods usually enjoy relatively low forgetting rate, they suffer from the drawback of linearly growing memory overheads resulted from task-specific masks as continual learning proceeds.
The other type, known as ``regularization-based methods''~\cite{kirkpatrick2017overcoming,huszar2017quadratic,zenke2017continual,nguyen2017variational,lee2017overcoming}, impose model regularization terms when training on the new task to prevent model parameters from deviating away from previously learned ones and thus alleviate forgetting issue.
They do not involve any task-specific masks but usually suffer more severe forgetting comparing to pruning-based ones.
The hard trade-off between forgetting prevention and memory efficiency for these two types of methods is because they only consider preventing forgetting on the parameter-level.

In this paper, we dive into the finer granularity of bit-level instead of parameter-level to investigate and address the forgetting problem in continual learning.
To start with, we study continual learning from information theory perspective, which is developed from the Bayesian inference framework that interprets supervised learning as an inference process over the model parameters from the training data.
For a continual learning (inference) process, where data seen by model accumulate as the model experiences more tasks, the inference on each model parameter should become increasingly certain.
To quantify this expected increase of certainty on model parameters given the streaming data, we exploit \textit{information gain}, which corresponds to the reduction of entropy on parameters estimated after experiencing new task data.
In this way, we consider continual learning as a recursive process of gaining information on model parameters, driven by the sequentially incoming task data.
From this viewpoint, forgetting can be understood as loss of information provided by previous task data on model parameters.

To more intuitively study information gain in continual learning context, we quantize model parameters (\eg to 20 bits) and view them in their bit representation, where values are represented by series of bits.
Initially, before learning on any tasks, a model parameter is free to update, which corresponds to all bits of the parameter being free-to-flip.
Shannon entropy, which is equivalent to the number of free-to-flip bits, is 20 at this point.
Next, after training on the first incoming task, information gain brought by the first task data, which is the reduction of Shannon entropy after learning the first task, thus corresponds to how many bits are now becoming certain.\footnote{As will be explained in Sec.~\ref{bit_fixing}, bits becoming certain from more significant bits (with higher bit positions) to less significant ones (with lower bit positions).}
Leaving these bits to continue to flip in the subsequent learning process means discarding the information provided by this task, and forgetting thus happens.
This motivates us to freeze these certain bits and prohibiting them from flipping in subsequent task iterations to preserve information gain provided by the first task.
By applying the same information gain-guided bit freezing after learning each subsequent task, information provided by each task can be preserved and forgetting can thus be prevented in continual learning.

We accordingly develop a Bit-Level Information Preserving (BLIP) method to tackle the forgetting issue in continual learning.
Given an incoming task, BLIP first trains a weight quantized neural network.
Then, BLIP estimates the information gain on the parameters brought by this task's data, which suggests how many bits to freeze.
The frozen bits are not allowed to flip in subsequent learning, preserving the information provided by this task.
This process is applied recursively with each new task to enable continual learning without forgetting.
We provide a simple overview of its workflow in Fig.~\ref{figure1}.

Unlike previous pruning-based methods, the memory overheads of our method is constant as BLIP only keeps track of how many bits to freeze in total for each parameter, without requiring any task-specific masks.
On the other hand, our method is a more precise and stronger way of regularizing parameters compared to previous regularization-based methods, which effectively prevents forgetting.
We validate our method across diverse settings of classification tasks and reinforcement learning tasks, and well prove that it performs on par with or better than previous state-of-the-arts.

This work makes following contributions: 
\begin{enumerate}[1)]
  \item We study continual learning through the lens of information theory and provide a new perspective on forgetting from information gain.
  \item We propose a novel approach called Bit-Level Information Preserving (BLIP), which quantizes model parameters and directly preserves information gain on model parameters brought by previously learned tasks via bit freezing. Comparing to prior works, our method enjoys advantages of both low forgetting rate and reasonable memory overheads.
  \item We conduct extensive experiments ranging from classification tasks to reinforcement learning tasks to demonstrate the effectiveness of our method.
  \item To the best of our knowledge, our work is the first to explore weight quantization for continual learning.
\end{enumerate}

\section{Related Work}
\textbf{Pruning Based Continual Learning}
These methods rely on network pruning to preserve the previously learned knowledge~\cite{fernando2017pathnet,serra2018overcoming,mallya2018packnet,mallya2018piggyback,ebrahimi2019uncertainty}.
In the training stage, these methods divide their model parameters into two sets: frozen set and free set. Only free parameters can change to adapt to learn the incoming task while frozen parameters stay unchanged in order to protect the previously learned knowledge.
After learning on a task, pruning is applied to free parameters to identify parameters that are important for this task.
A task-specific mask on model parameters is then saved to mark the parameters frozen after learning this task.
During the evaluation stage, choosing which parameter to use is conditioned on the given task and task-specific masks.

\textbf{Regularization Based Continual Learning}
These methods are mostly based on a Bayesian Inference framework
\cite{kirkpatrick2017overcoming,zenke2017continual,nguyen2017variational,lee2017overcoming,ebrahimi2019uncertainty,ahn2019uncertainty}.
Under this framework, posterior distribution on model parameters after learning all previous tasks is viewed as prior when learning a new task. 
From this point of view, knowledge preserving can be achieved by regularizing parameter posterior of the new task to avoid deviating drastically from the prior (which is the posterior given previous tasks).
This is normally achieved by adding penalty terms in the optimization objective.

\textbf{Replay Based Continual Learning}
These methods save representative examples of previous tasks in a replay buffer of limited size \cite{lopez2017gradient,chaudhry2018efficient,isele2018selective, rolnick2019experience,zhang2019variational,ebrahimi2020remembering} or train a generative model to generate samples of previous tasks \cite{shin2017continual,van2018generative,mundt2019unified,xiang2019incremental}.
During training of a new incoming task, data in replay buffer or produced by the trained generative model are then used to constrain the model to perform consistently on previous tasks. 

\textbf{Dynamic Architecture Based Continual Learning}
These methods \cite{aljundi2017expert,yoon2017lifelong,li2019learn} normally enlarge the model dynamically to adapt new incoming tasks. During inference, different model components are applied conditioned on the given or inferred task identity.

\textbf{Parameter Quantization} 
The parameter quantization technique has been widely applied to accelerating the deep neural networks \cite{bengio2013estimating,zhou2016dorefa}.
These works usually apply aggresive quantization on model parameters (\eg 1/2/4 bits) for efficiency consideration, which causes performance degradation on models.
However, in our work, we quantize parameters simply to view them in their bit representations and thus develop our method.
Therefore, we only slightly quantized the parameter (\eg to 20 bits).

\section{Preliminaries}
\textbf{Problem Setup and Assumptions}
We consider learning a total of $T$ tasks sequentially with a deep neural network model.
Data of these $T$ tasks are denoted by $\{\mathcal{D}_{1},\mathcal{D}_{2},...,\mathcal{D}_{T}\}$ with  $\mathcal{D}_{t}=\{\mathcal{X}_{t},\mathcal{Y}_{t}\}$ for $t \in \{1,2,...,T\}$. Here $\mathcal{X}_{t}$ denotes the set of raw data and $\mathcal{Y}_{t}$ denotes the corresponding label set. Throughout the continual learning process, when the model is learning task $t$, it is not allowed to refer to $\mathcal{D}_{1},\ldots, \mathcal{D}_{t-1}$.
In order to estimate the information gain on model parameters provided by each task's data, we adopt the Bayesian inference framework, which interprets model parameters as variables to be inferred given the training data.

We assume posteriors on all model parameters are mutually independent Gaussian distributed similar to previous works \cite{kirkpatrick2017overcoming,huszar2017quadratic}.
Therefore, without loss of generality, we study a single scalar model parameter $\theta$ in the following to illustrate our method.

\textbf{Notations and Definitions}
We denote the posterior on the model parameter $\theta$ after sequentially learning on task $1$ to $t$ as $p(\theta|{\mathcal{D}_{1},\mathcal{D}_{2},\ldots,\mathcal{D}_{t}})$ with the shorthand being $p(\theta_{0:t})$.
The value of $\theta$ after learning on tasks $1$ to $t$ is denoted as $\theta_{0:t}^{*}$.
Loss on $\mathcal{D}_{t}$ when learning on task $t$ is denoted as $l(D_{t},\theta)$.
The model prediction output on a data sample $x$ given the model parameter $\theta$ is denoted as $p_{\theta}(x)$, and $p_{\theta}(y|x)$ denotes the prediction probability of $x$'s label being $y$.
Since we rely on Fisher information of $\theta$ to estimate information gain and develop our method, we denote Fisher information estimated over $\mathcal{D}_{t}$ as $F_{t}$ and the one over $\mathcal{D}_{1} \ldots \mathcal{D}_{t}$ as $F_{0:t}$\footnote{$F_{0}$ corresponds to prior of Fisher information before learning any task, which is a hyper-parameter in our method.}. 
We denote the quantization function by $Q$.
The result of quantizing $\theta$ to $N$ bits is denoted by $Q(\theta,N)$.
This means $Q(\theta,N)$ has a total of $N$ bits in its bit representation.
Correspondingly, the Shannon entropy on $Q(\theta,N)$ is defined as
\begin{equation}
\label{Shannon}
   H(Q(\theta,N)) = -\sum_{i=1}^{2^{N}} P(Q(\theta,N)=q_{i}) \log_{2} P(Q(\theta,N)=q_{i}),    
\end{equation}
where $q_{i}$ denotes $Q(\theta,N)$'s $i$-th possible value.

\section{Bit-Level Information Preserving}
\label{method_dev}
\subsection{Motivation and Method Overview}
Continual learning can be regarded as a continual inference process over model parameters given streaming task data under Bayesian inference framework.
As data seen by the model accumulates, the inference certainty on model parameters shall gradually increase.
In order to quantify such increase of certainty, we introduce the concept of \textit{information gain}, which amounts to the reduction of entropy.

To better understand the intuition of information gain in the context of continual learning, we study a model parameter $\theta$ quantized to $N$ bits via $Q$, and view it in its bit representation, whose value is represented by a series of binary bits.
Before learning any tasks, $\theta$ is free to update, which is equivalent to all $N$ bits of $Q(\theta,N)$ being free-to-flip.
Approximately speaking, each free-to-flip bit in $Q(\theta,N)$ flips with an equal probability, and the Shannon entropy can thus correspond to the total number of freely flipped bits, which is $N$ in this case.

After learning an incoming task, the posterior on $\theta$ becomes more peaky and concentrated, which corresponds to some bits of $Q(\theta,N)$ becoming certain and being not able to flip.
As will be elaborated, information provided by this task's data on $Q(\theta,N)$ immediately indicates how many bits of $Q(\theta,N)$ become certain from more significant bits to less significant bits.
To preserve information provided by this task and thus prevent forgetting, a natural solution is to freeze these certain bits and not allow them to flip.

Motivated by this intuition of information gain, we develop our Bit-Level Information Preserving (BLIP) method, which iteratively conducts information gain-guided bit freezing to prevent forgetting in the continual learning. 
Specifically, without loss of generality, we consider learning the $t$-th task to illustrate our method.
We denote $n_{t}$ as the number of bits frozen after learning the $t$-th task and $n_{0:t}=\sum_{i=1}^{i=t}n_{i}$ is the total number of bits frozen after all previous $t$ tasks.
Before learning on task $t$, $\theta_{0:t-1}^{*}$ and $n_{0:t-1}$ are saved to do bit freezing, and $F_{0:t-1}$ is saved to help estimate information gain provided by task $t$ data.

To learn task $t$, the proposed method quantizes the neural network weights and trains it on $\mathcal{D}_{t}$.
Throughout the training process, the first $n_{0:t-1}$ bits of $Q(\theta,N)$ are frozen by BLIP to prevent forgetting on previous tasks and the rest $N-n_{0:t-1}$ bits can flip to adapt on $\mathcal{D}_{t}$.
After the training is completed, BLIP estimates information gain provided by $\mathcal{D}_{t}$, which implies how many additional bits become certain.
Finally, bits that become certain after learning $\mathcal{D}_{t}$ are frozen in the subsequent learning to  prevent forgetting on task $t$.

We further elaborate on the details of above three steps in the following subsections.

\subsection{Training on Task $t$}
\label{train_w_quant}
To start with, we describe our model parameter quantization scheme and explain how to optimize $\theta$ on $\mathcal{D}_{t}$ while keeping $Q(\theta,N)$'s first $n_{0:t-1}$ bits frozen.

Our quantization function $Q$ is defined as follows:
\begin{equation}
\label{quantize_func}
  Q(\theta, N)=\frac{\round{(2^{N} \min(\max(\theta, -1+\frac{1}{2^{N+1}}), 1-\frac{1}{2^{N+1}}))}}{2^N},
\end{equation}
where $\round{\cdot}$ rounds the number to its nearest integer.
For the quantized model, we rely on Straight Through Estimator (STE) \cite{bengio2013estimating}, which approximates the gradient of the quantization function $Q$ by $\frac{\partial Q(\theta,N)}{\partial \theta} = 1$, to optimize the model parameter $\theta$.
Therefore, $\theta$ is directly updated with the gradient descent by
\begin{equation}
  \theta := \theta - \alpha \frac{\partial l(\mathcal{D}_{t},\theta)}{\partial Q(\theta, N)},
\end{equation}
where $\alpha$ is the learning rate.

Next, to keep the first $n_{0:t-1}$ bits of $Q(\theta,N)$ unchanged,
we clip the value of $\theta$ after the above gradient descent updates.
Specifically, based on the definition of $Q$ in~Eqn.~\eqref{quantize_func}, keeping $Q(\theta,N)$'s first $n_{0:t-1}$ bits frozen is equivalent to restricting $\theta$ to be within an interval centered around $c_{t-1}=Q(\theta_{0:t-1}^{*},n_{0:t-1})$ with the radius of $\frac{1}{2^{n_{0:t-1}}}$.

Therefore, the first $n_{0:t-1}$ can be kept frozen by the following clipping operation:
\begin{equation}
\label{clipping}
    \theta := \min(\max(\theta,c_{t-1}-\frac{1}{2^{n_{0:t-1}}}), c_{t-1}+\frac{1}{2^{n_{0:t-1}}}).
\end{equation}

The above gradient descent and clipping steps are applied repeatedly until the training loss $l$ converges on $\mathcal{D}_{t}$.

\subsection{Estimating Information Gain}
\label{info_gain}
After training on $\mathcal{D}_{t}$, we estimate the Information Gain (IG) on $Q(\theta,N)$ provided by $\mathcal{D}_{t}$, which is the reduction of the Shannon entropy on $Q(\theta,N)$ after training on $\mathcal{D}_{t}$:
\begin{equation}
\label{discrete_ig}
  IG(Q(\theta,N),\mathcal{D}_{t}) = H(Q(\theta_{0:t-1},N)) - H(Q(\theta_{0:t},N)).
\end{equation}
According to this definition and the definition of the Shannon entropy Eqn.~\eqref{Shannon}, estimating information gain requires the knowledge of $p(\theta_{0:t-1})$ and $p(\theta_{0:t})$, which are the posterior on $\theta$ before and after learning on $\mathcal{D}_{t}$ respectively.

Unfortunately, the posterior on $\theta$ is intractable and has to be approximated by Laplace's approximation \cite{mackay2003information}.
This approximation states that given a dataset $\mathcal{D}=\{\mathcal{X},\mathcal{Y}\}$, $p(\theta|\mathcal{D})$ can be approximated by $\mathcal{N}(\theta_{\mathcal{D}}^{*}, (m F_{\mathcal{D}}(\theta_{\mathcal{D}}^{*}))^{-\frac{1}{2}})$, where $\theta_{\mathcal{D}}^{*}$ is $\theta$'s value after learning on $D$, $m$ is the number of samples used to estimate Fisher information, and $F_{\mathcal{D}}(\theta)$ is \textit{Fisher information} estimated over $\mathcal{D}$:
\begin{equation}
\label{fisher_def}
    F_{\mathcal{D}}(\theta) = \mathbb{E}_{x \sim \mathcal{X}, y \sim p_{\theta}(x)} \left[\left(\frac{\partial \ln p_{\theta}(y|x)}{\partial \theta}\right)^2\right].
\end{equation}

By assuming we use the same number of samples to estimate Fisher information for each task, $p(\theta_{0:t-1})$ and $p(\theta_{0:t})$ can be approximated by $\mathcal{N}(\theta^{*}_{0:t-1},(mt F_{0:t-1})^{-\frac{1}{2}})$ and $\mathcal{N}(\theta^{*}_{0:t},(m(t+1)F_{0:t})^{-\frac{1}{2}})$ respectively.
$F_{0:t-1}$, as mentioned before, is previously saved after learning task $1$ to $t$, and only $F_{0:t}$ is unknown.
To obtain $F_{0:t}$, we first estimate $F_{t}$ over $\mathcal{D}_{t}$ by Eqn.~\eqref{fisher_def}.
Then, with $F_{0:t-1}$ and $F_{t}$, $F_{0:t}$ can be obtained by
\begin{equation}
\label{Fisher_recursion}
    F_{0:t} = \frac{t F_{0:t-1} + F_{t}}{t+1}.
\end{equation}

With the posterior approximation introduced above and the definition of information gain Eqn.~\eqref{discrete_ig}, the information gain provided by $\mathcal{D}_{t}$ can be estimated by
\begin{equation}
\begin{aligned}
\label{IG_4_2_3}
  IG(Q(\theta,N),\mathcal{D}_{t})
  &\approx \frac{1}{2}\log_{2}m(t+1)F_{0:t} - \frac{1}{2}\log_{2}mtF_{0:t-1} \\
  &= \frac{1}{2}\log_{2}\frac{t F_{0:t-1} + F_{t}}{t F_{0:t-1}}.
\end{aligned}
\end{equation}
Refer to appendix for detailed elaboration on Eqn.~\eqref{IG_4_2_3}.

\subsection{Information Gain-Guided Bit Freezing}
\label{bit_fixing}
With the estimated information gain, we now investigate which bits shall be frozen to prevent forgetting on task $t$ in the subsequent learning.
Approximately speaking, each free bit of $Q(\theta,N)$ flip with equal probability, and the Shannon entropy on $Q(\theta,N)$ is thus the number of bits free-to-flip in total.
Information gain, which is the reduction of the Shannon entropy after learning $\mathcal{D}_{t}$, is thus approximately equivalent to how many more bits of $Q(\theta,N)$ become certain throughout the process.
Leaving these certain bits to continue to flip means discarding the information gain provided, which directly motivates us to freeze these certain bits.
Therefore, the number of additional bits to freeze $n_{t}$ can be obtained by
\begin{equation}
    n_{t} = \ceil{IG(Q(\theta,N),\mathcal{D}_{t})}.
\end{equation}
In addition, since the posterior on $\theta$ is a Gaussian distribution and only becomes peaky and concentrated locally, bits in $Q(\theta,N)$ become certain starting from more significant bits (with higher bit positions) to less significant bits (with lower bit positions).
In practice, we also clip $n_{t}$ between $0$ and $N-n_{0:t-1}$ so that the number of bits frozen in total does not surpass $N$.
Therefore, by freezing the first $n_{t}$ bits of the remaining $N-n_{0:t-1}$ bits of $Q(\theta,N)$, our method can specifically preserve the information gain provided by $\mathcal{D}_{t}$ and thus prevent forgetting on task $t$.

In the end, the number of bits frozen in total $n_{0:t}$, current Fisher information $F_{0:t}$, as well as current parameter value $\theta_{0:t}^{*}$ are then saved for the next task iteration.
By recursively preserving information gain provided by each task, our method can mitigate forgetting in continual learning.
A detailed summary for our algorithm is shown in Alg.~\ref{algo_blip}.

\begin{algorithm}
\caption{Bit-Level Information Preserving (BLIP)}
\label{algo_blip}
\begin{algorithmic}[1]

\State {$\theta$ $\gets$ {randomly init parameters}}
\State {$N \gets$ {pre-defined quantization bits}}
\State {$c \gets 0$ {, } $S \gets 0$ {, } $F \gets$ {hyper parameter $F_{0}$}}
\State {$L \gets$ {loss function, } $\alpha \gets$ {learning rate}}
\State {$Q \gets$ {quantization function}}
\For{$t \gets 1$ to $T$}
    \State {obtain task $t$ dataset $D_{t}=\{\mathcal{X},\mathcal{Y}\}$}
    \State {$\theta \gets$ \Call{train}{$\theta,c,S,D,L$}}
    \State {$IG, F_{post} \gets $\Call{estimateInfoGain}{$\theta$,$\mathcal{X}$, $F$, $t$}}
    \State {$n_{t} \gets \min{(\max(\ceil{IG},0), N-S)}$}
    \State {$c \gets Q(\theta,S)$ {, } $S \gets S + n_{t}$ {, } $F \gets F_{post}$}
\EndFor

\Function{train}{$\theta, c, S, D, L, Q, \alpha$}
    \While{loss $L$ not converged on $D$}
        \State {$\theta$ $\gets$ $\theta - \alpha\nabla_{\theta}L(D,Q(\theta))$}
        \State {$\theta$ $\gets$ $\min(\max(\theta,c-\frac{1}{2^{S}}), c+\frac{1}{2^{S}})$}
    \EndWhile
    \State {return $\theta$}
\EndFunction

\Function{estimateInfoGain}{$\theta, \mathcal{X}, F$, $t$}
    \State {$F_{t} \gets 0$}
    \For{each data point $x$ in $\mathcal{X}$}
        \State {sample $y \sim p_{\theta}(x)$}
        \State {$F_{t} \gets F_{t} + \left(\frac{\partial \log p_{\theta}(y|x)}{\partial\theta}\right)^{2}$}
    \EndFor
    \State{$F_{t} \gets \frac{F_{t}}{|\mathcal{X}|}$}
    \State {$F_{post} \gets t F + F_{t}$}
    \State {$IG \gets \frac{1}{2}\log_{2}\frac{F_{post}}{tF}$}
    \State{return $IG$ {, } $F_{post}$}
\EndFunction
\end{algorithmic}
\end{algorithm}

\section{Experiments}
In this section, we report experiment results of our method.
In Sec.~\ref{exp_setup}, we introduce experimental setups.
Next, we demonstrate the effectiveness of our method in continual image classification settings of various scales and continual reinforcement learning in Sec.~\ref{cls_exp} and Sec.~\ref{reinforce_exp} respectively.
Finally, in Sec.~\ref{capacity_usage}, based on our RL agent, we visualize the bit freezing process throughout continual learning to further see how BLIP works.

\subsection{Experimental Setups}
\label{exp_setup}
\textbf{Evaluation Metrics}
We evaluate all methods with two widely used metrics: average accuracy (ACC) and backward transfer (BWT) \cite{lopez2017gradient}.
ACC is the average accuracy over all learned tasks at the end of the continual learning process, while BWT quantifies how learning on new tasks affects model performance on previously learned tasks.
Formally, let the total number of learned tasks be $T$, and $A_{i,j}$ denote the accuracy on the $j$-th task after sequentially learning the first $i$ tasks.
ACC is defined as
\begin{equation}
    ACC = \frac{1}{T}\sum_{i=1}^{i=T}A_{T,i},
\end{equation}
and BWT is defined as
\begin{equation}
    BWT = \frac{1}{T-1}\sum_{i=1}^{i=T}(A_{T,i}-A_{i,i}).
\end{equation}
Negative BWT implies forgetting on previous tasks while positive means learning on new tasks can even help with the performance on previous tasks.
For each of the two metrics, higher values  indicate better performance.

\textbf{Benchmarks and Models}
We target the setting of continual learning with disjoint tasks, where task identity is given during both training and evaluation. The following continual image classification benchmarks are used to evaluate our method:
MNIST-5 \cite{ebrahimi2019uncertainty,ebrahimi2020adversarial}, PMNIST \cite{kirkpatrick2017overcoming,serra2018overcoming}, Alternating Cifar10/100 \cite{serra2018overcoming}, 20-Split mini-ImageNet \cite{zhang2019variational,chaudhry2019tiny}, Sequence of 5 tasks \cite{ebrahimi2020adversarial}.
Reinforcement Learning agents are evaluated over a sequence of six Atari environments \cite{mnih2013playing}.
For MNIST-based settings, we adopt a two layer perceptron with 1,200 hidden units as \cite{ebrahimi2019uncertainty}.
For other classification settings, we use AlexNet \cite{krizhevsky2017imagenet} for all methods.
For the reinforcement learning setting, we use a PPO agent \cite{schulman2017proximal} with three convolution layers and one fully connected layer.

\textbf{Baselines}
For image classification tasks, we compare our method with naive baselines of direct sequential fine-tuning (FT), fine-tuning classifier with backbone fixed after the first task (FT-FIX), and previous methods including
EWC \cite{kirkpatrick2017overcoming,huszar2017quadratic},
VCL \cite{nguyen2017variational}
IMM \cite{lee2017overcoming},
LWF \cite{li2017learning},
A-GEM \cite{chaudhry2018efficient},
HAT \cite{serra2018overcoming},
UCB \cite{ebrahimi2019uncertainty}, and
ACL \cite{ebrahimi2020adversarial}.
Among these methods, HAT and ACL incur linearly growing memory overheads due to task specific parameters.
A variant of ACL, which is ACL with replay buffer (ACL-R), using replay buffer to further boost forgetting prevention.
In addition, UCB and VCL are based on Bayesian neural networks \cite{blundell2015weight}, which have much higher computation cost than normal networks during training as several times of Monte-Carlo sampling are needed over model parameters in one iteration of gradient descent.
For IMM, we use its "mode" version \cite{lee2017overcoming}.
For EWC, we use its online variant proposed in \cite{huszar2017quadratic}.

\textbf{Hyperparameters}
For classification tasks, our code is based on released implementations of \cite{serra2018overcoming}.
Plain Stochastic Gradient Descent with batch size $32$ and initial learning rate $0.05$ is used for optimization.
Learning rate decays by a factor of $3$ if validation loss plateaus for 5 consecutive epochs.
Training stops when learning rate is below $1 \times 10^{-4}$ or we have iterated over 200 epochs.
For reinforcement learning, the PPO agent is trained with initial learning rate $2.5 \times 10^{-4}$, entropy regularization coefficient  $0.01$.
We sample totally 10 millions steps in each environment to train our agents.
Model parameters are quantized to 20 bits for BLIP. In addition, we ablate the choice for $F_{0}$ in appendix.

\subsection{Continual Learning for Image Classification}
\label{cls_exp}
\subsubsection{MNIST-5 and P-MNIST}
\label{mnist_5_text}
We start with a relatively simple setting of MNIST-5 where the ten MNIST classes are equally split into five separate tasks for the model to learn sequentially.
In addition, we evaluate our method on another MNIST-based continual learning setting, Permuted MNIST.
In this setting, the first task is the standard MNIST classification while the following nine tasks are created by permuting pixels of MNIST images with nine different schemes.
Results on MNIST-5 and P-MNIST are shown in Tab.~\ref{mnist5_res} and Tab.~\ref{pmnist_res} respectively.
From these results, it can be seen that our method achieves on par performance with state-of-the-art methods such as UCB, HAT, and is significantly better than other methods in terms of ACC and BWT.
However, as mentioned, UCB and HAT achieve these results with either much higher computation cost or undesired growing memory overheads during training, while our method is free from these hazards.

\begin{table}[t]
\caption{Results on MNIST-5. MO is memory overheads complexity. Results denoted by ($^{*}$) are provided by \cite{ebrahimi2019uncertainty}.}
\label{mnist5_res}
\centering
\begin{tabular}{ |p{1.5cm}|p{1.4cm}|p{1.4cm}|p{1.2cm}|}
    \hline
    Methods& BWT(\%) & ACC(\%) & MO\\
    \hline
    VCL$^{*}$       & -0.56 & 98.20 & $O(1)$ \\
    IMM$^{*}$       & -11.20& 88.54 & $O(1)$ \\
    EWC$^{*}$       & -4.20 & 95.78 & $O(1)$ \\
    HAT$^{*}$       &  0.00 & 99.59 & $O(T)$ \\
    UCB$^{*}$       &  0.00 & 99.63 & $O(1)$ \\
    \hline
    \textbf{BLIP}      &  0.01 & 99.64 & $O(1)$ \\
    \hline
\end{tabular}
\end{table}

\begin{table}
\centering
\caption{Results on PMNIST. Results denoted by ($^{*}$) are provided by \cite{ebrahimi2019uncertainty}.}
\label{pmnist_res}
\begin{tabular}{|p{1.5cm}|p{1.4cm}|p{1.4cm}|p{1.2cm}|}
    \hline
    Methods& BWT(\%) & ACC(\%) & MO\\
    \hline
    LWF$^{*}$        &-31.17 & 65.65 & $O(1)$ \\
    IMM$^{*}$        & -7.14 & 90.51 & $O(1)$ \\
    HAT$^{*}$        & 0.03  & 97.34 & $O(T)$ \\
    UCB$^{*}$        & 0.03  & 97.42 & $O(1)$ \\
    \hline
    \textbf{BLIP}    & -0.21 & 97.31 & $O(1)$ \\
    \hline
\end{tabular}
\end{table}

\begin{table}
    \centering
    \caption{Results on Alternating Cifar10/100.}
    \label{alt_cifar_res}
    \begin{tabular}{ |p{1.5cm}|p{1.4cm}|p{1.4cm}|p{1.2cm}|  }
    \hline
    Methods& BWT(\%) & ACC(\%) & MO\\
    \hline
    FT                &-14.23 &  67.94 & 0 \\
    FT-FIX            &  0.00 &  43.78 & 0 \\
    LWF                & -40.88 & 44.27 & $O(1)$ \\
    IMM                & -16.47 &  66.43 & $O(1)$ \\
    EWC                &  -5.5 &  72.11 & $O(1)$ \\
    HAT                & -0.02 &  80.22 & $O(T)$ \\
    \hline
    \textbf{BLIP}      & -0.43 &  74.70 & $O(1)$ \\
    \hline
    \end{tabular}
\end{table}

\subsubsection{Alternating Cifar10/100}
Next, we evaluate our method on a more challenging setting of Alternating Cifar10/100.
Cifar10 and Cifar100 datasets are both divided equally into five different tasks respectively, with two classes in each Cifar10 task and twenty classes in each Cifar100 task.
The model is required to learn sequentially  the combined ten tasks.
Our results are presented in Tab.~\ref{alt_cifar_res}.
In this setting, one can observe that the gap in terms of forgetting prevention between pruning-based method (HAT) and regularization-based methods (LWF, IMM, EWC) is very large.
Our method BLIP, which suffers no memory overheads during evaluation, greatly surpasses the compared regularization-based methods and makes a significant step towards continual learning being free from forgetting and growing memory overheads.

\subsubsection{20-Split mini-ImageNet}
To further show the effectiveness of our method in larger scale scenario, we conduct experiments in the setting of 20-Split mini-ImageNet, where the 100 classes of mini-ImageNet are equally split into 20 tasks. Results are shown in Tab.~\ref{mini-imagenet-main}. Our splitting scheme is exactly the same as \cite{ebrahimi2020adversarial}. Model size of each method is kept similar for fair comparison. From the results, we can see that BLIP achieves the highest ACC. HAT has better BWT than BLIP by using task-specific parameters while ACL-R boosts BWT with replay buffer.

\begin{table}[ht]
\caption{Experiment Results on 20-Split mini-ImageNet. RB is size of replay buffer. Results are averaged over 5 random seeds; mean $\pm$ std is reported. Results denoted by ($^{\dagger}$) are provided by \cite{ebrahimi2020adversarial}.}
\label{mini-imagenet-main}
\centering
\begin{tabular}{|p{1.3cm}|p{2.0cm}|p{2.0cm}|p{1.4cm}|}
\hline
Methods&BWT (\%)&ACC (\%)& RB (MB)\\
\hline
LWF               &-45.93 $\pm$ 1.05&29.30 $\pm$ 0.64 & - \\
A-GEM$^{\dagger}$   &-15.23 $\pm$ 1.45 &52.43 $\pm$ 3.10 & 110.1 \\
HAT$^{\dagger}$     &-0.04 $\pm$ 0.03 &59.45 $\pm$ 0.05 & -\\
ACL$^{\dagger}$     &-3.71 $\pm$ 1.31 &57.66 $\pm$ 1.44 & -  \\
ACL-R$^{\dagger}$   &0.00 $\pm$ 0.00 & 62.07 $\pm$ 0.51 & 8.5  \\
\hline
\textbf{BLIP}     &-1.05 $\pm$ 0.42 & 65.69 $\pm$ 0.87 & - \\
\hline
\end{tabular}
\vspace{-5mm}
\end{table}

\begin{table}[ht]
\caption{Experiment Results on Sequence of 5 tasks. MS is Model Size. Results are averaged over 5 random seeds; mean $\pm$ std is reported. Results denoted by ($^{\dagger}$) are provided by \cite{ebrahimi2020adversarial}.}
\label{seq_of_5}
\centering
\begin{tabular}{|p{1.2cm}|p{2.0cm}|p{2.0cm}|p{1.4cm}|}
\hline
Methods& BWT (\%) & ACC (\%) & MS (MB) \\
\hline
FT                &-34.16 $\pm$ 9.00&65.31 $\pm$ 7.18& 16.97 \\
FT-FIX            &0.00 $\pm$ 0.00&71.57 $\pm$ 4.55& 16.97 \\
LWF               &-61.14 $\pm$ 5.92&42.93 $\pm$ 4.59& 16.97 \\
IMM               &-21.56 $\pm$ 5.46&73.39 $\pm$ 4.20& 16.97 \\
UCB$^{\dagger}$   &-1.34 $\pm$ 0.04&76.34 $\pm$ 0.12& 16.4  \\
ACL$^{\dagger}$   &-0.01 $\pm$ 0.15&78.55 $\pm$ 0.29& 16.5  \\
\hline
\textbf{BLIP}      & -0.13 $\pm$ 0.324& 82.87 $\pm$ 1.43&  16.97 \\
\hline
\end{tabular}
\vspace{-5mm}
\end{table}

\begin{figure*}
  \centering
  \includegraphics[width=15cm]{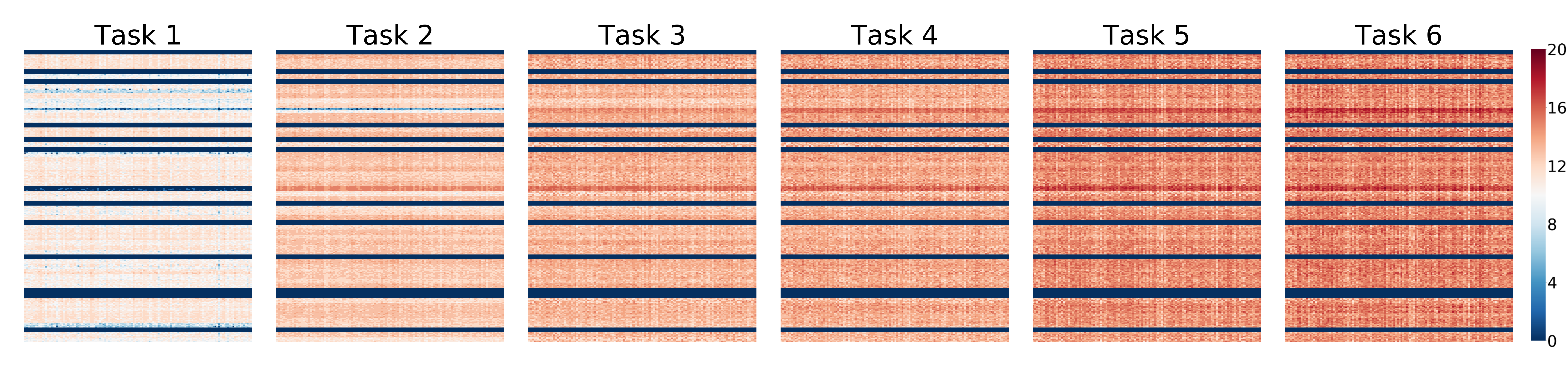}
  \caption{\textbf{Bit Freezing Visualization} We visualize bit freezing process of our RL agent's third convolution layer parameters during continual learning.
  Each pixel in a heat map represents the number of frozen bits of the corresponding entry (parameter) in weight matrix of the convolution layer.
  Each parameter has a total of 20 bits.
  From darker blue to darker red denotes more bits being frozen.
  Visualization of other layers are shown in appendix.
  }
  \label{fig_cap_usage}
  \vspace{-3mm}
\end{figure*}

\begin{figure*}
  \centering
  \includegraphics[width=14.5cm]{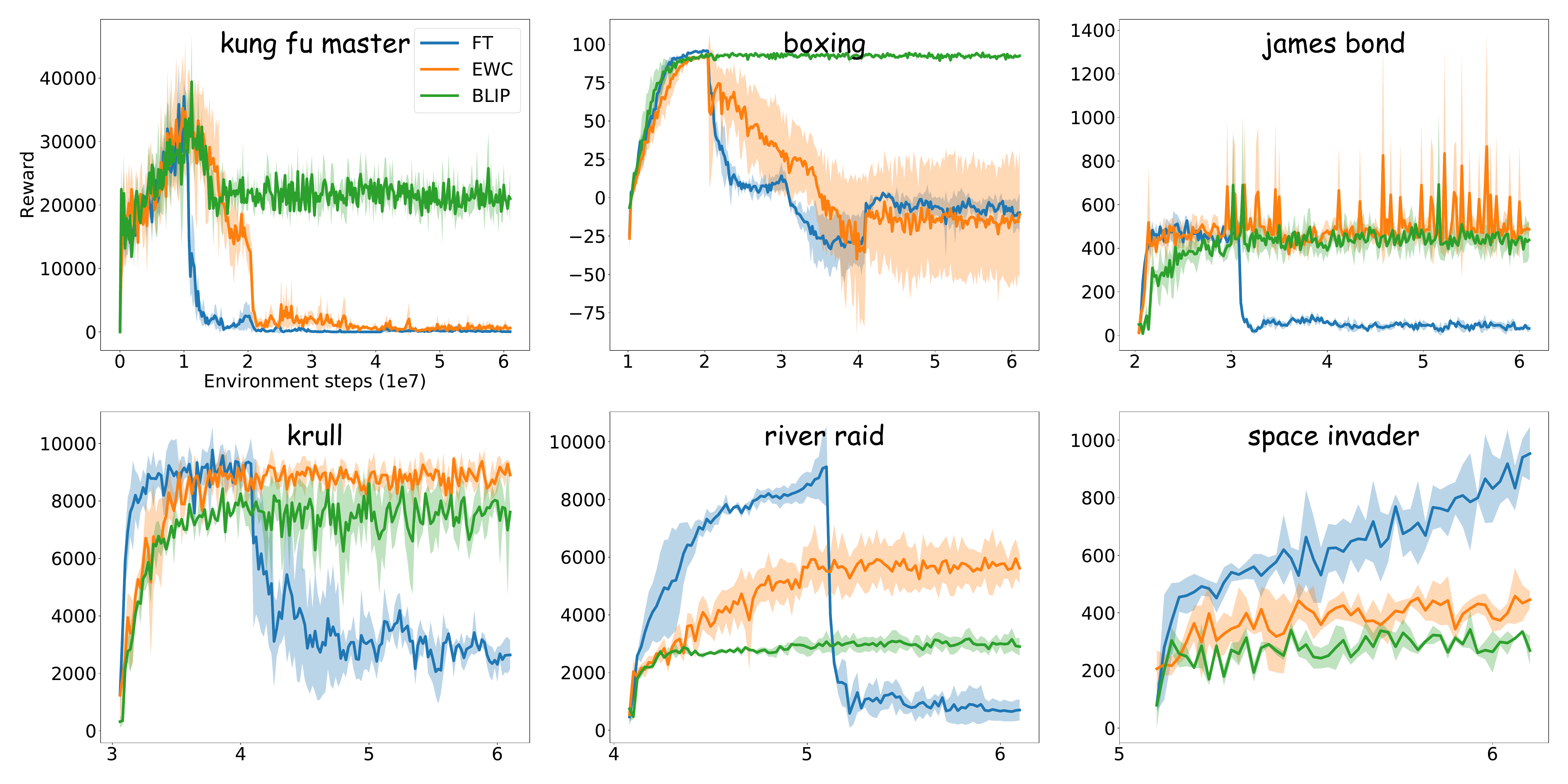}
  \caption{\textbf{Rewards on each task}. Each task is trained for 10 million environment steps, and 60 million environment steps are thus used in total. Tasks are learned left to right and top to bottom. Results are averaged over 3 random seeds.
  }
  \label{fig_each_task_reward}
  \vspace{-3mm}
\end{figure*}

\begin{figure}
  \centering
  \includegraphics[width=6.5cm]{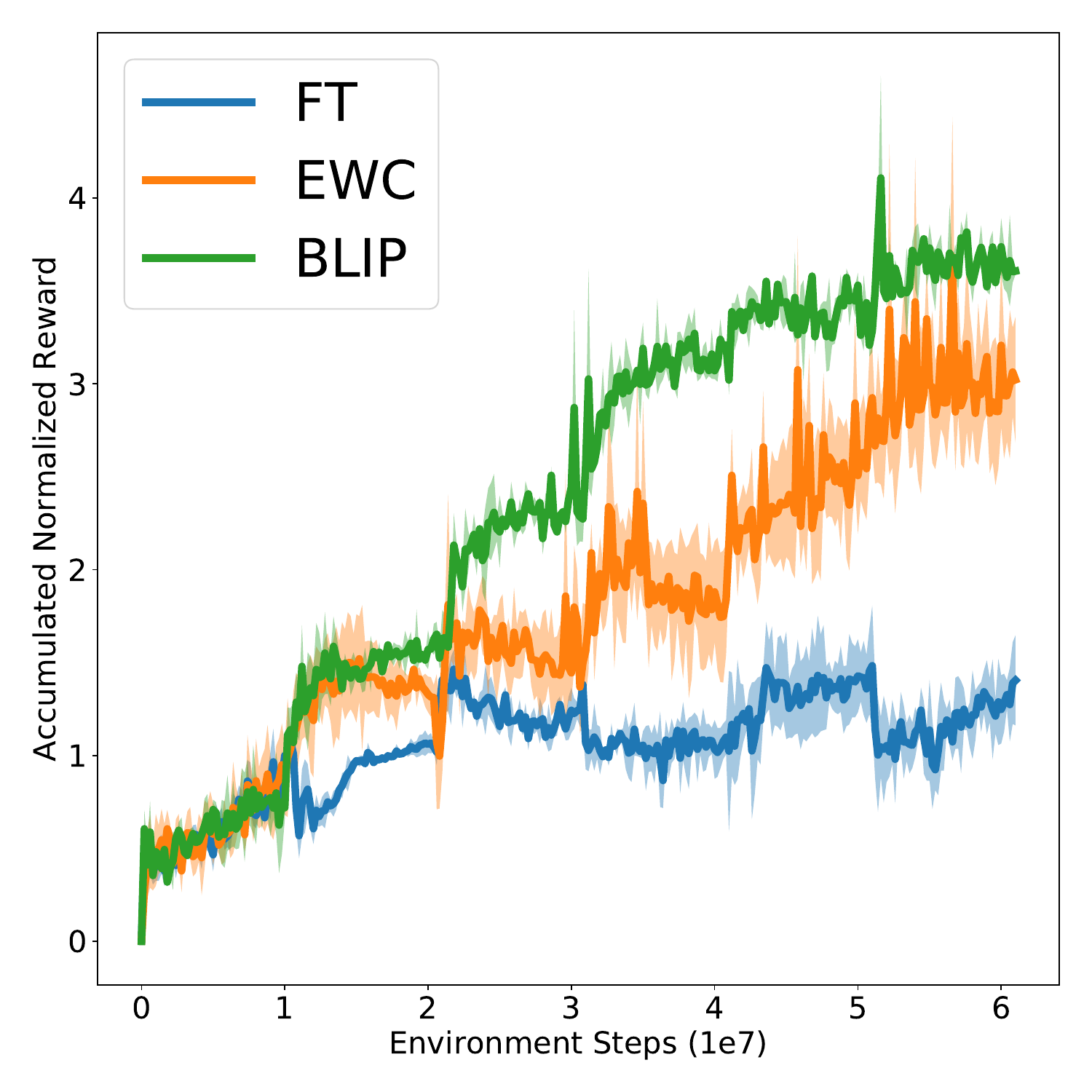}
  \caption{\textbf{Accumulated Normalized Rewards}.
  }
  \label{fig_overall_reward}
  \vspace{-4.5mm}
\end{figure}

\subsubsection{Sequence of 5 tasks}
To further demonstrate our method's ability to learn sequentially on more diverse domains, we evaluate it by learning on Cifar10, notMNIST, MNIST, SVHN, FashionMNIST.
Our results are averaged over 5 random seeds as \cite{ebrahimi2020adversarial}.
From the results, one can observe that our method considerably surpasses previous state-of-the-arts such as UCB and ACL in terms of ACC.
Because ACL relies on given task information and saved task-specific modules, it can achieve better BWT than BLIP.
When doing a fair comparison with other methods, BLIP can achieve much better BWT.

\subsection{Continual Reinforcement Learning}
\label{reinforce_exp}
Our continual reinforcement learning benchmark requires the agent to sequentially learn to play the following six Atari games: kung fu master, boxing, james bond, krull, river raid, space invaders.

We first demonstrate our results on each task separately in Fig.~\ref{fig_each_task_reward}, where we plot rewards of trained agents on each task from the point they start to train on the task to the ending of the whole continual learning process.
From this figure, we can see that the agent trained by FT suffers from catastrophic forgetting as its reward on one task drastically decreases after shifting to learning on the following tasks.
For the agent trained with EWC, although forgetting is alleviated and performance can be retained on the last four tasks, its performance on the first two tasks can not sustain through the continual learning process and decreases dramatically.
Different from these two, our method can remember to play all tasks in this continual learning setting.

Next, to compare the performance of each method in an overall manner, we plot the accumulated reward, which is the sum of rewards on all learned tasks at each environment step, in Fig.~\ref{fig_overall_reward}.
We normalize rewards on each task to the scale of 0 to 1.
Therefore, if an agent can both learn and remember all six tasks well, its accumulated normalized reward shall steadily increase from 0 to 6 in the continual learning process.
On the contrary, the accumulated normalized reward should oscillate around 1 after training the first task if an agent constantly forgets to perform on previously learned tasks, which is the case for FT in the figure.
As illustrated, our method BLIP outperforms EWC in terms of the accumulated normalized reward in the whole continual learning process.

\subsection{Visualizing Bit Freezing and Related Discussion}
\label{capacity_usage}
To further understand the underlying working mechanism of BLIP through the whole continual learning process, we visualize the corresponding bit freezing process in Fig.~\ref{fig_cap_usage}.
In this figure, we show how bits are gradually frozen for parameters of our RL agent after learning each task.
As can be observed, more and more bits of model parameters are frozen to anchor the learned knowledge as continual learning proceeds.
In addition, we can see less bits are frozen after learning the last two tasks comparing to the first four tasks, which means not many bits of the model parameters have been effectively used to adapt to the last two tasks.
This corresponds to the phenomenon that BLIP does not adapt well on the last two tasks in Fig.~\ref{fig_each_task_reward}.
We leave solving this problem for future work.

Previous neurobiology evidences suggest that human brain may mitigate forgetting by synaptic consolidation\cite{yang2009stably,yang2014sleep}, which is the process of strengthening synapses that are crucial for certain knowledge or skills.
In the context of artificial neural networks, where synapses (connections between neurons) are implemented as weight parameters, BLIP mimics the process of synaptic consolidation via bit freezing.
Concretely, as continual learning proceeds, more bits are gradually frozen as shown in Fig.~\ref{fig_cap_usage}, which corresponds to the process of synapses being strengthened as more skills and knowledge are learned.

\section{Conclusion}
In this work, we study continual learning through lens of information theory and provide a new perspective on forgetting from information gain.
Based on this perspective, we propose a novel approach called Bit-Level Information Preserving, which dives into bit-level and directly preserves information gain on model parameters provided by previously learned tasks via bit freezing. 
We conduct extensive experiments to show that our method can successfully adapt to continual classification settings of various scales as well as continual reinforcement learning settings.

\section*{Acknowledgment}
The authors would like to thank Jun Hao Liew, Kuangqi Zhou, Minda Hu, Bingyi Kang, Zihang Jiang and reviewers for their helpful discussions and feedbacks. This work was partially supported by AISG-100E-2019-035, MOE2017-T2-2-151, NUS\_ECRA\_FY17\_P08 and CRP20-2017-0006.


{\small
\bibliographystyle{ieee_fullname}
\bibliography{egbib}
}

\clearpage
\onecolumn
\section{Appendix}
\subsection{Detailed Elaboration on Information Gain Estimation}
In this section, we elaborate on \textbf{how to obtain Eqn.(8) mentioned in the main text}.

To start with, recall that we have introduced in the main text that $p(\theta_{0:t-1})$ and $p(\theta_{0:t})$ can be approximated by:
\begin{equation}
\label{post_t_1}
    p(\theta_{0:t-1}) = \mathcal{N}(\theta^{*}_{0:t-1},(tmF_{0:t-1})^{-\frac{1}{2}})
\end{equation}
and
\begin{equation}
\label{post_t}
    p(\theta_{0:t}) = \mathcal{N}(\theta^{*}_{0:t},((t+1)mF_{0:t})^{-\frac{1}{2}})
\end{equation}
respectively.
In addition, relation between $F_{0:t-1}$ and $F_{0:t}$ is:
\begin{equation}
    F_{0:t} = \frac{t  F_{0:t-1} + F_{t}}{t+1}.
\end{equation}

Next, recall the definition of information gain on the quantized parameter $Q(\theta,N)$ is:
\begin{equation}
    \label{IG_def}
    IG(Q(\theta,N), \mathcal{D}_{t}) = H(Q(\theta_{0:t-1},N)) - H(Q(\theta_{0:t},N)).
\end{equation}
To calculate $IG(Q(\theta,N), \mathcal{D}_{t})$, we first connect $H(Q(\theta,N))$ with $h(\theta)$, where $h(\theta)$ is differential entropy defined as:
\begin{equation}
    h(\theta) = -\int p(\theta)\ln{p(\theta)}d\theta.
\end{equation}
$H(Q(\theta,N))$ and $h(\theta)$ can be connected by the following lemma.
\begin{lemma}
Consider a random variable $X$ with density function $p(x)$ with support of $[-1+\frac{1}{2^{N}}, 1-\frac{1}{2^{N}}]$ and assume $p(x)\ln{p(x)}$ is Riemann integrable. If we quantize $X$ by $N$ bits with $Q$ defined in the main text, then we have:
\begin{equation}
\label{dis_cont_transfer}
    H(Q(X, N)) \approx \frac{1}{\ln2}h(X) + N - 1.
\end{equation}
\end{lemma}
\begin{proof}
Firstly, we can divide the support of $X$, which is $[-1+\frac{1}{2^{N}}, 1-\frac{1}{2^{N}}]$,
into $2^{N}-1$ bins with equal length of $\delta_{N} = \frac{1}{2^{N-1}}$.
Next, denote the center of the $i$-th bin as $x_{i}$ and we have: $x_{i} = \frac{i}{2^{N-1}} - 1$.
We write $P(Q(X,N)=x_{i})$ with shorthand of $P_{i}(Q(X,N))$ and we have:
\begin{equation}
\label{MVT}
\begin{aligned}
    P_{i}(Q(X,N)) &= \int_{x_{i}-\frac{\delta_{N}}{2}}^{x_{i}-\frac{\delta_{N}}{2}}{p(x)dx}
    \;\;\;\text{(property of probability density function)}\\
    &\approx p(x_{i})\delta_{N}
    \;\;\;\text{(Mean Value Theorem as $\delta_{N} \rightarrow 0$.)}
\end{aligned}
\end{equation}

With the above Eqn.~\eqref{MVT}, we then rewrite the Shannon entropy of $Q(X,N)$ as:
\begin{equation}
\begin{aligned}
    H(Q(X,N)) &= -\sum_{i=1}^{2^{N}-1} P_{i}(Q(X,N)) \log_{2} P_{i}(Q(X,N)) \\
    &\approx -\sum_{i=1}^{2^{N}-1} p(x_{i})\delta_{N}\log_{2}p(x_{i})\delta_{N}
    \;\;\;\text{(applying \eqref{MVT})}\\
    &= -\sum_{i=1}^{2^{N}-1} \delta_{N}p(x_{i})\log_{2}p(x_{i})
    -\sum_{i=1}^{2^{N}-1}\delta_{N}p(x_{i})\log_{2}\delta_{N} \\
    &\approx \frac{1}{\ln2}h(X) + (N - 1)
    \;\;\;\text{(Riemann integrable as $\delta_{N} \rightarrow 0$;$\delta_{N}=\frac{1}{2^{N-1}}$)}
\end{aligned}
\end{equation}
\end{proof}

Therefore, the information gain can be rewritten as:
\begin{equation}
    IG(Q(\theta,N),\mathcal{D}_{t}) = \frac{1}{\ln{2}}(h(\theta_{0:t-1}) - h(\theta_{0:t})).
\end{equation}
With the posterior approximation Eqn.~\ref{post_t_1} and Eqn.~\ref{post_t}, $h(\theta_{0:t-1})$ and $h(\theta_{0:t})$ can be calculated in closed-form by $\frac{1}{2} - \frac{1}{2}\ln{(2\pi{m t F_{0:t-1}})}$ and $\frac{1}{2} - \frac{1}{2}\ln{(2\pi{m (t+1) F_{0:t}})}$ respectively.
That means that $IG(Q(\theta,N),\mathcal{D}_{t})$ can be approximated by:
\begin{equation}
\label{IG_appendix}
\begin{aligned}
 IG(Q(\theta,N),\mathcal{D}_{t})
 &=\frac{1}{\ln{2}}(h(\theta_{0:t-1}) - h(\theta_{0:t})) \\
 &\approx \frac{1}{2\ln{2}}\ln\frac{m (t+1) F_{0:t}}{m t F_{0:t-1}} \\
 &= \frac{1}{2}\log_{2}\frac{t F_{0:t-1} + F_{t}}{t F_{0:t-1}}.
\end{aligned}
\end{equation}

In this way, we derive the Eqn.(8) mentioned in the main text.

\subsection{Implementation Details}
\subsubsection{Parameter Range}
To perform parameter quantization, we normally have to constraint parameters to be within a certain interval. In the main text, to more conveniently elaborate our method, we assume model parameters are constrained within $[-1,1]$ and then quantize them accordingly. However, in real scenarios, parameters of different layers normally distributed in different manner. Therefore, we constraint model parameters in the interval of $[-\frac{C}{\sqrt{n}}, \frac{C}{\sqrt{n}}]$, where $C$ is a pre-defined hyper-parameter, and $n = num\_input\_dimension$ for fully connected layers and $n=kernel\_size \times kernel\_size \times num\_input\_channel$ for convolution layers. $C$ is set to be 20 for the mini-ImageNet experiments. and 6 for all the other experiments.
This strategy is inspired by previous literature on model parameter initialization.

\subsubsection{Trainig Strategy}
According to the main text, when trainig on a new task, Straight Through Estimator (STE) is used to perform quantization aware training. However, since we quantize the model to 20 bits in all our experiments and the difference is insignificant between whether or not using STE. Therefore, we do not quantize parameters during training and only quantize parameters before doing bit freezing.

\subsubsection{Information Gain}
Empirically, we find that using the following formulation:
\begin{equation}
    IG(Q(\theta,N),\mathcal{D}_{t}) = \frac{1}{2}\log_{2}\frac{t F_{0:t-1} + F_{t}}{(t+1) F_{0:t-1}}
\end{equation}
which is slightly different from Eqn.~\eqref{IG_appendix}, can produce better results. Therefore, we adopt this slightly modified version to estimate information gain.

For more details, please refer to our released implementation.

\subsection{More Detailed Experiment Results}
\subsubsection{More Detailed 20-split mini-ImageNet Results}
In this section, we add comparison of model size for the 20-split mini-ImageNet experiment in Tab.~\ref{mini-imagenet}. All models are variants of AlexNet. From the table, we show that the model size of all methods are kept approximate to ensure fair comparisons.

\subsubsection{Experiments on ResNet}
In addition, to demonstrate the effectiveness of our method on models with Batch Norm layers and residual connections, we evaluate our method on ResNet-18 model with 20-split mini-ImageNet. The result is shown in Tab.~\ref{resnet}.

\subsubsection{Discussion and Ablation Study of $F_{0}$}
As mentioned in the main text, $F_{0}$ is an important hyper-parameter for our method. According to our method, smaller $F_{0}$ corresponds to more bits being frozen when learning subsequent tasks and vice versa. Therefore, by setting $F_{0}$ to be small, our method tends to suffer less forgetting while being less capable of adapting new tasks and vice versa.
We ablate using different $F_{0}$ with our 20-split mini-ImageNet experiment in Tab.~\ref{ablate_F_0}. Through this ablation study, we can see that setting $F_{0}=5 \times 10^{-16}$ can best balance between preventing forgetting (BWT) and adapting new tasks (ACC).

\begin{table}[H]
\caption{Experiment Results on 20-Split mini-ImageNet. RB is size of replay buffer, MS is model size. Results are averaged over 5 runs; mean $\pm$ std is reported. Results denoted by ($^{\dagger}$) are provided by the ACL paper.}
\label{mini-imagenet}
\centering
\begin{tabular}{|p{1.6cm}|p{2.0cm}|p{2.0cm}|p{1.4cm}|p{1.4cm}|}
\hline
Methods&BWT (\%)&ACC (\%)& RB (MB) & MS (MB)\\
\hline
LWF               &-45.93 $\pm$ 1.05&29.30 $\pm$ 0.64 & - & 104.1 \\
A-GEM$^{\dagger}$             &-15.23 $\pm$ 1.45 &52.43 $\pm$ 3.10 & 110.1 & 102.6 \\
HAT$^{\dagger}$               &-0.04 $\pm$ 0.03 &59.45 $\pm$ 0.05 & - & 123.6 \\
ACL$^{\dagger}$               &-3.71 $\pm$ 1.31 &57.66 $\pm$ 1.44 & - & 113.1 \\
ACL-R$^{\dagger}$             &0.00 $\pm$ 0.00 & 62.07 $\pm$ 0.51 & 8.5 & 113.1 \\
\hline
\textbf{BLIP}     &-1.05 $\pm$ 0.42 & 65.69 $\pm$ 0.87 & - & 104.78\\
\hline
\end{tabular}
\end{table}

\begin{table}[H]
\caption{Experiment Results on 20-split mini-ImageNet with ResNet-18 and AlexNet. MS is Model Size. Arch is model architecture. Results are averaged over 5 random seeds; mean $\pm$ std is reported.}
\label{resnet}
\centering
\begin{tabular}{|p{1.2cm}|p{2.0cm}|p{2.0cm}|p{2.0cm}|p{1.4cm}|}
\hline
Methods& Arch& BWT (\%) & ACC (\%) & MS (MB) \\
\hline
BLIP & AlexNet & -1.05 $\pm$ 0.42 & 65.69 $\pm$ 0.87 & 104.78 \\
BLIP & ResNet & -0.72 $\pm$ 0.46 & 65.94 $\pm$ 1.36 & 42.76 \\
\hline
\end{tabular}
\end{table}

\begin{table}[H]
\caption{Ablation study on $F_{0}$ with AlexNet and 20-split mini-ImageNet. All results are averaged over 5 random seeds. mean $\pm$ std is reported.}
\label{ablate_F_0}
\centering
\begin{tabular}{|p{3.0cm}|p{3.0cm}|p{3.0cm}|}
\hline
$F_{0}$& BWT (\%) & ACC (\%) \\
\hline
$1 \times 10^{-14}$ & -3.60 $\pm$ 0.59 & 65.23 $\pm$ 0.87  \\
$5 \times 10^{-15}$ & -3.01 $\pm$ 0.76 & 65.10 $\pm$ 1.16  \\
$1 \times 10^{-15}$ & -1.55 $\pm$ 0.53 & 65.76 $\pm$ 0.81  \\
$5 \times 10^{-16}$ & -1.05 $\pm$ 0.42 & 65.69 $\pm$ 0.87  \\
$1 \times 10^{-16}$ & -0.26 $\pm$ 0.32 & 64.78 $\pm$ 0.96  \\
$5 \times 10^{-17}$ & -0.17 $\pm$ 0.29 & 64.29 $\pm$ 0.82  \\
\hline
\end{tabular}
\end{table}

\subsection{More Visualization on Bit Freezing Process}
In this section, we provide more results on the bit freezing process visualization on different layers of our PPO agent in Fig~\ref{bit_freezing}. From the results, the similar phenomenon of more bits getting frozen as mentioned in the main text can be observed. From the heatmap, we can also see that for some parameters, no bits are frozen throughout the whole continual learning process. This means that no information gain is observed on these parameters. This phenomenon might correspond to the nature of ReLU activation function, whose gradient is 0 if the input is negative. However, it is beyond the scope of this work and we do not provide further discussion on this phenomenon.

In addition, this bit freezing process is further visualized in histogram in Fig.~\ref{bit_freezing_hist}.

\begin{figure}[H]
  \centering
  \includegraphics[width=16.5cm]{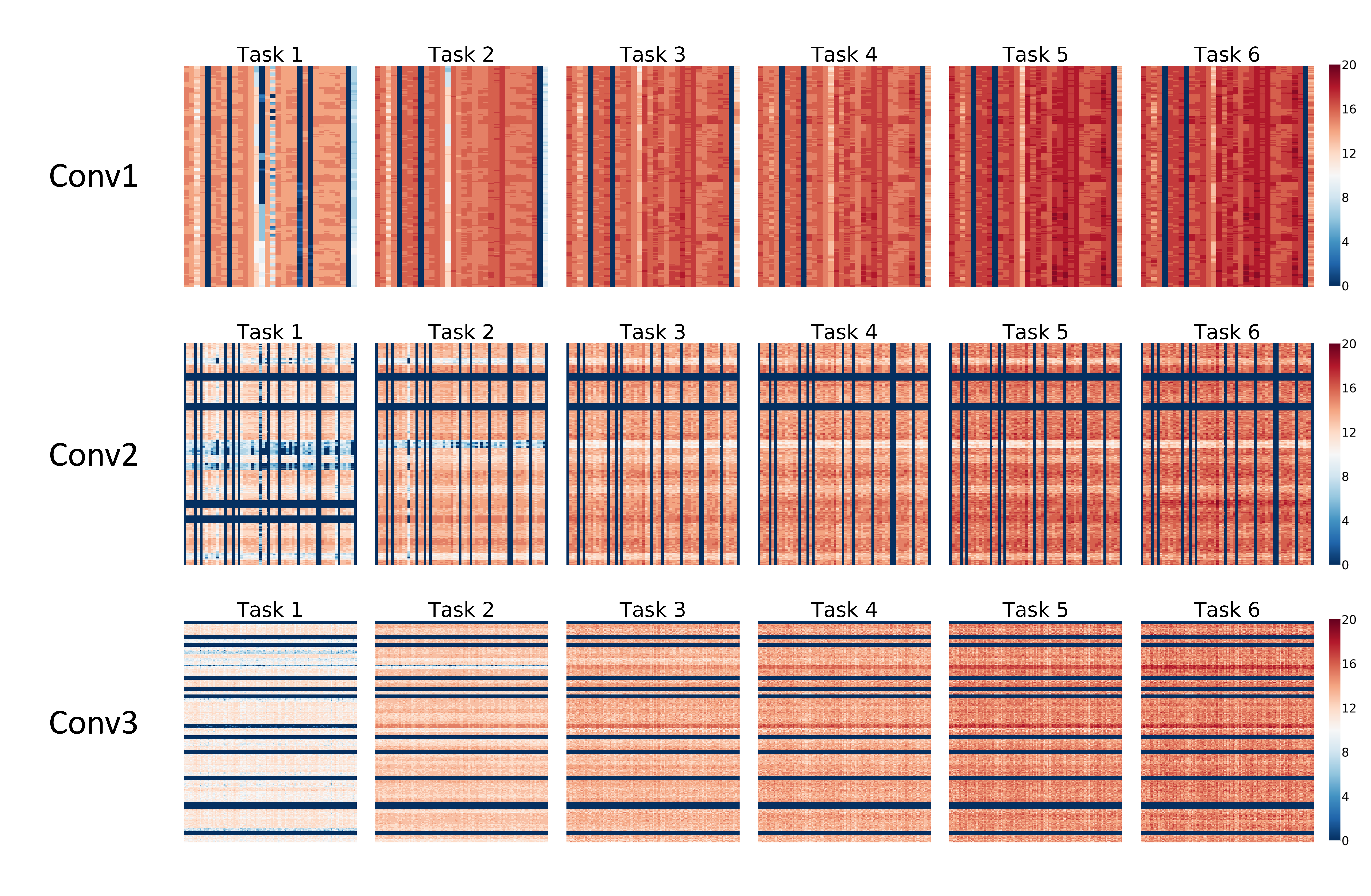}
  \caption{\textbf{Bit freezing visualization} We visualized bit freezing process as continual learning proceeds on all the convolution layers of our PPO agent. Each pixel in a heat map represents the number of frozen bits of the corresponding entry (parameter) in weight matrix of a convolution layer.
  Each parameter has a total of 20 bits.
  From darker blue to darker red denotes more bits being frozen.}
  \label{bit_freezing}
\end{figure}

\begin{figure}[H]
    \centering
    \includegraphics[width=8.5cm]{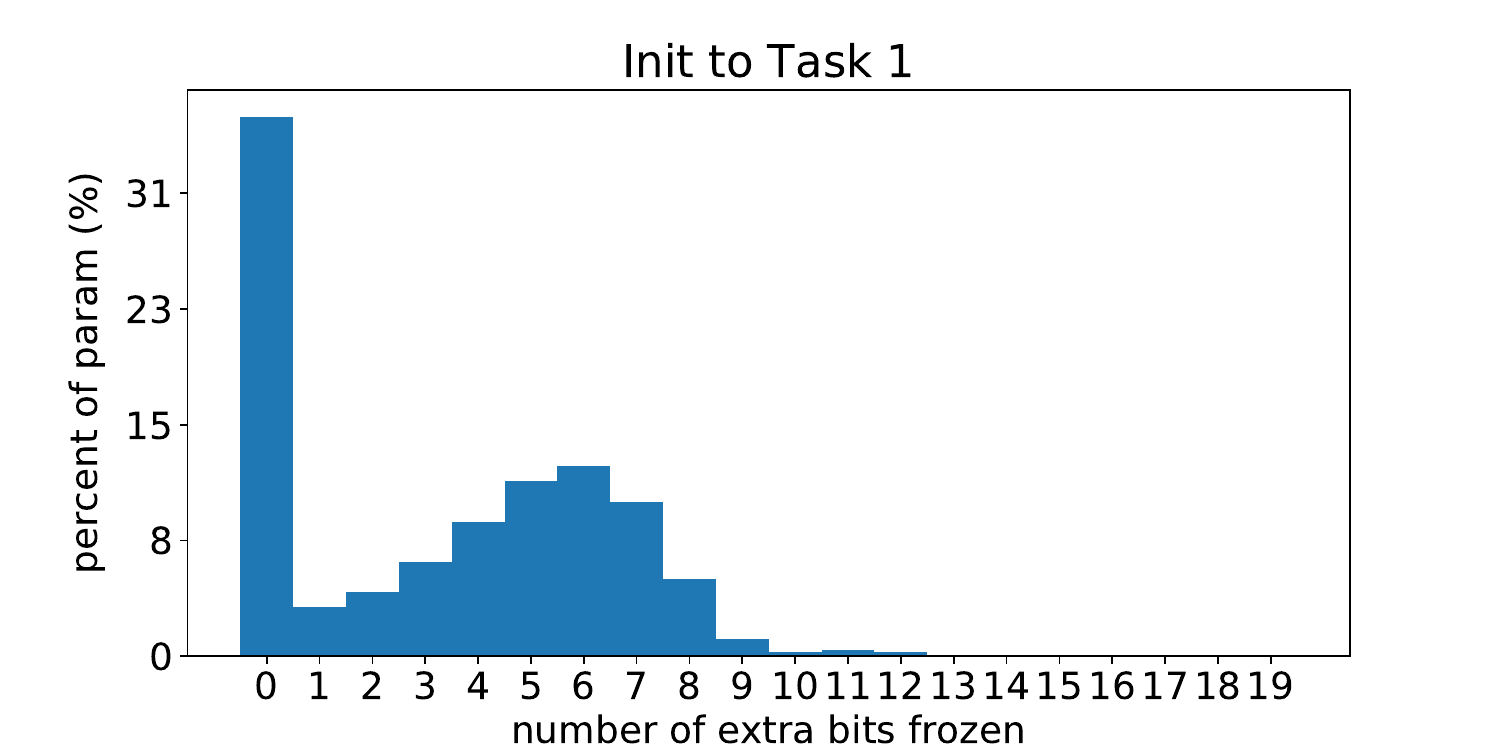}
    \includegraphics[width=8.5cm]{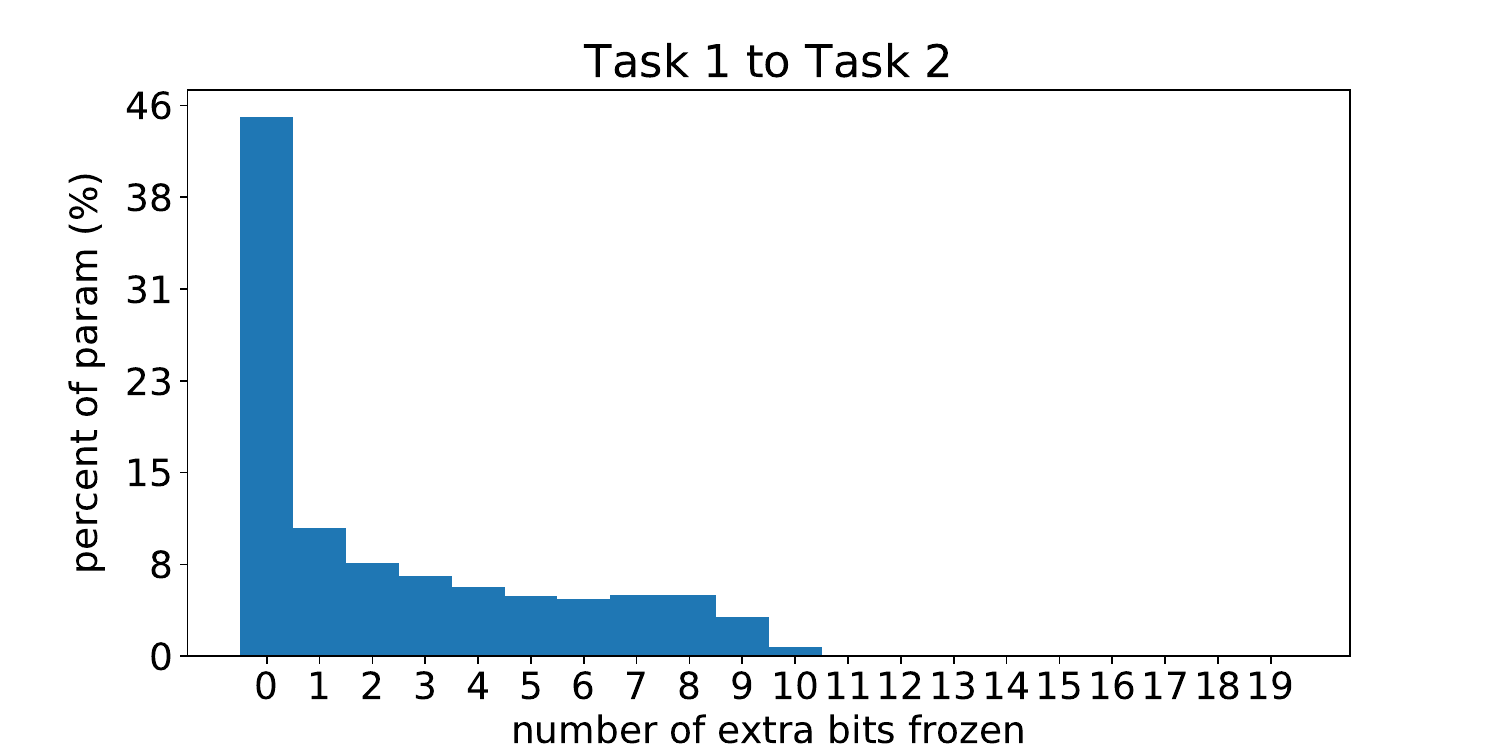}
    \includegraphics[width=8.5cm]{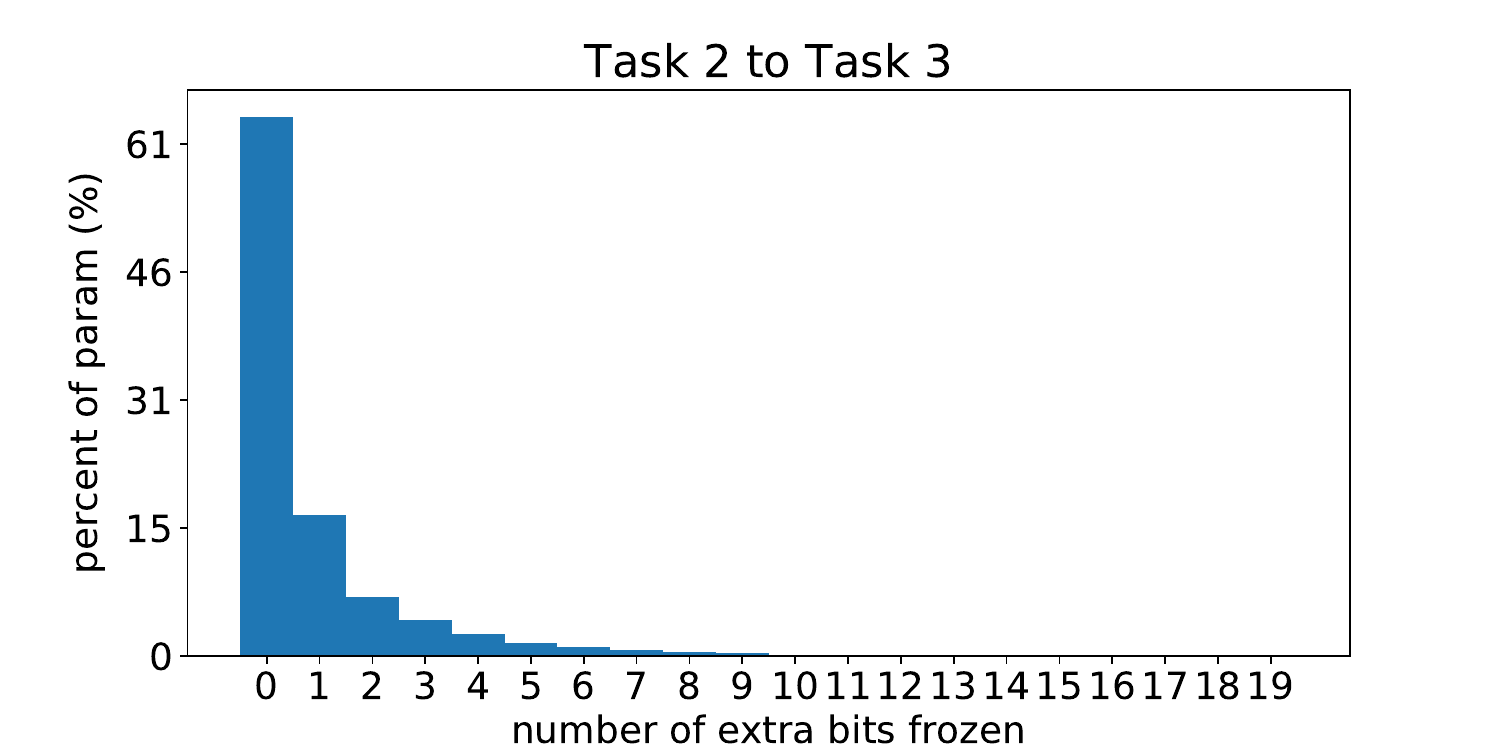}
    \includegraphics[width=8.5cm]{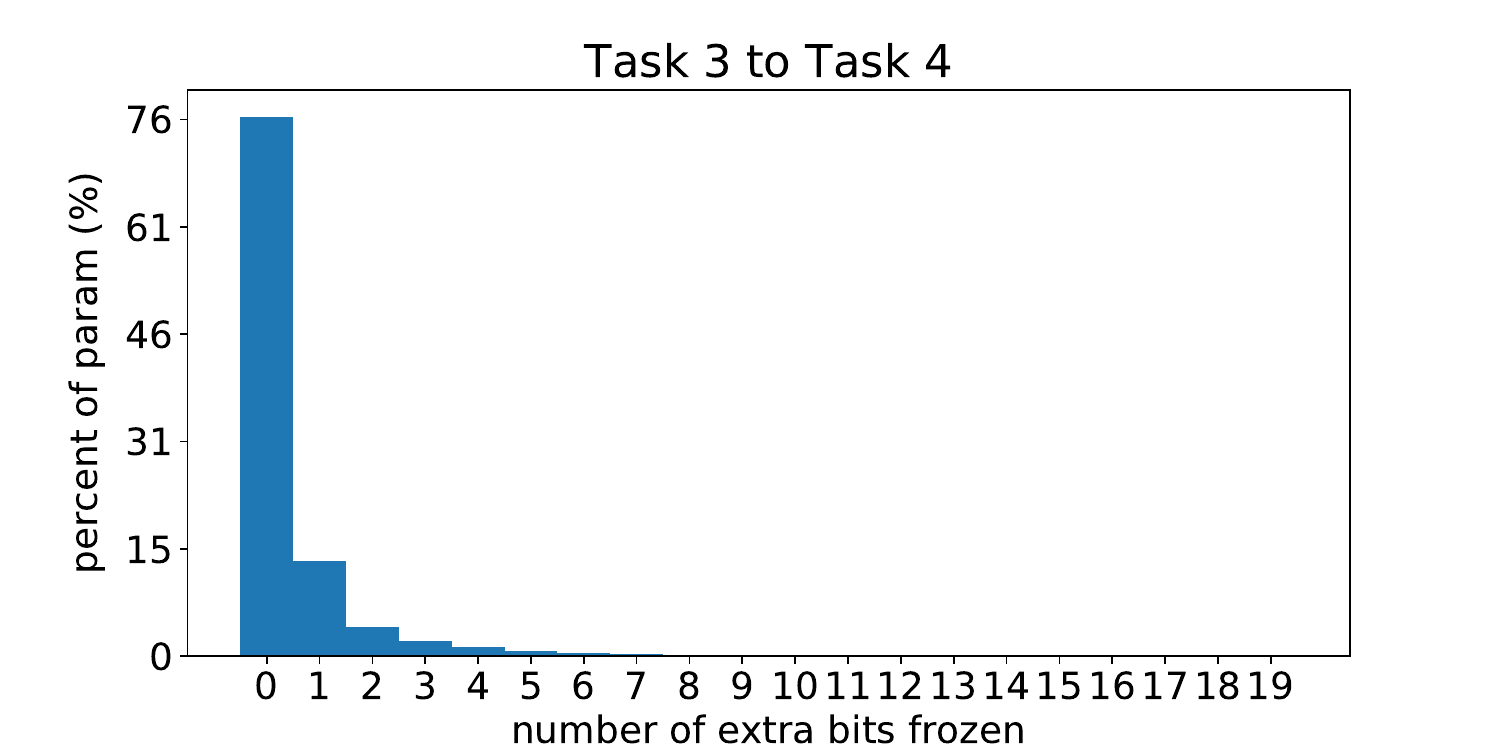}
    \includegraphics[width=8.5cm]{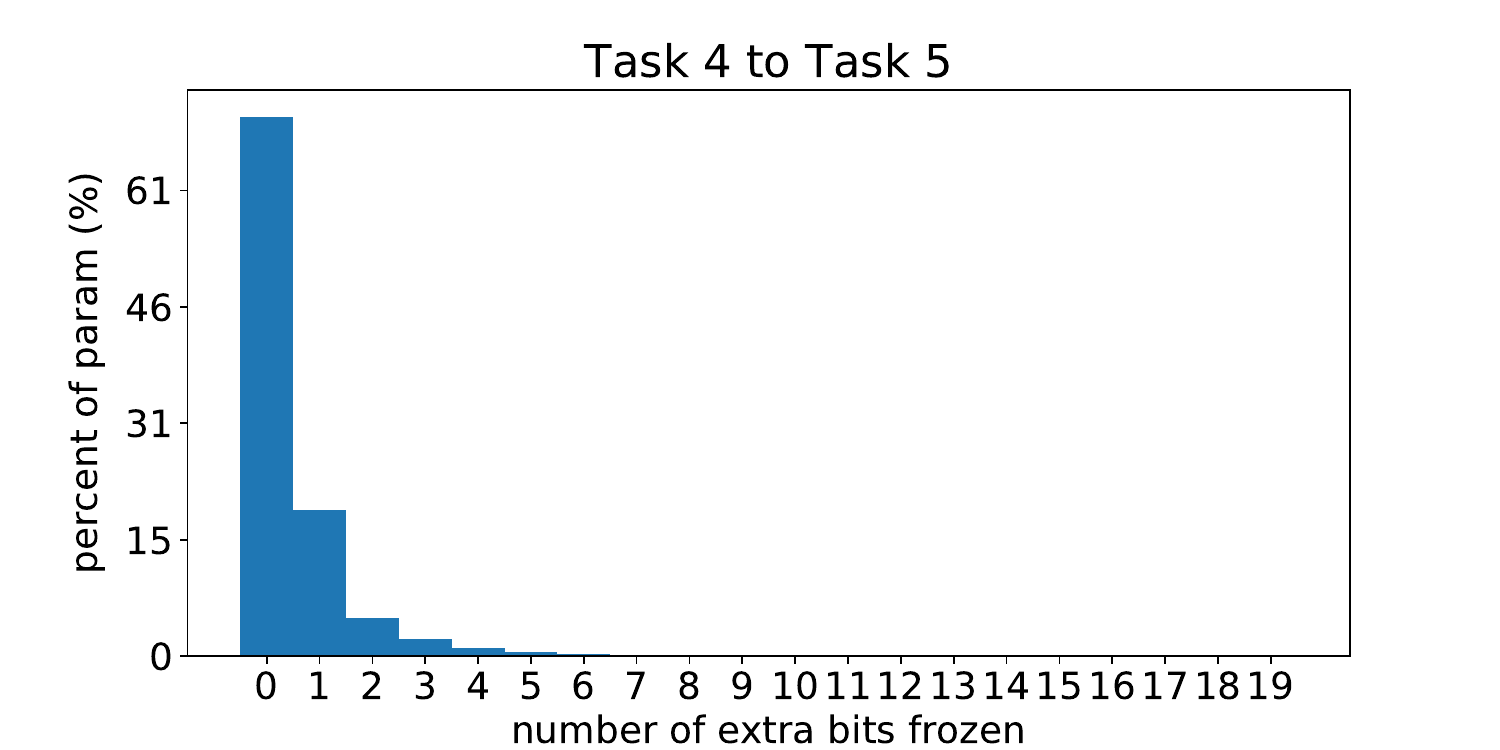}
    \includegraphics[width=8.5cm]{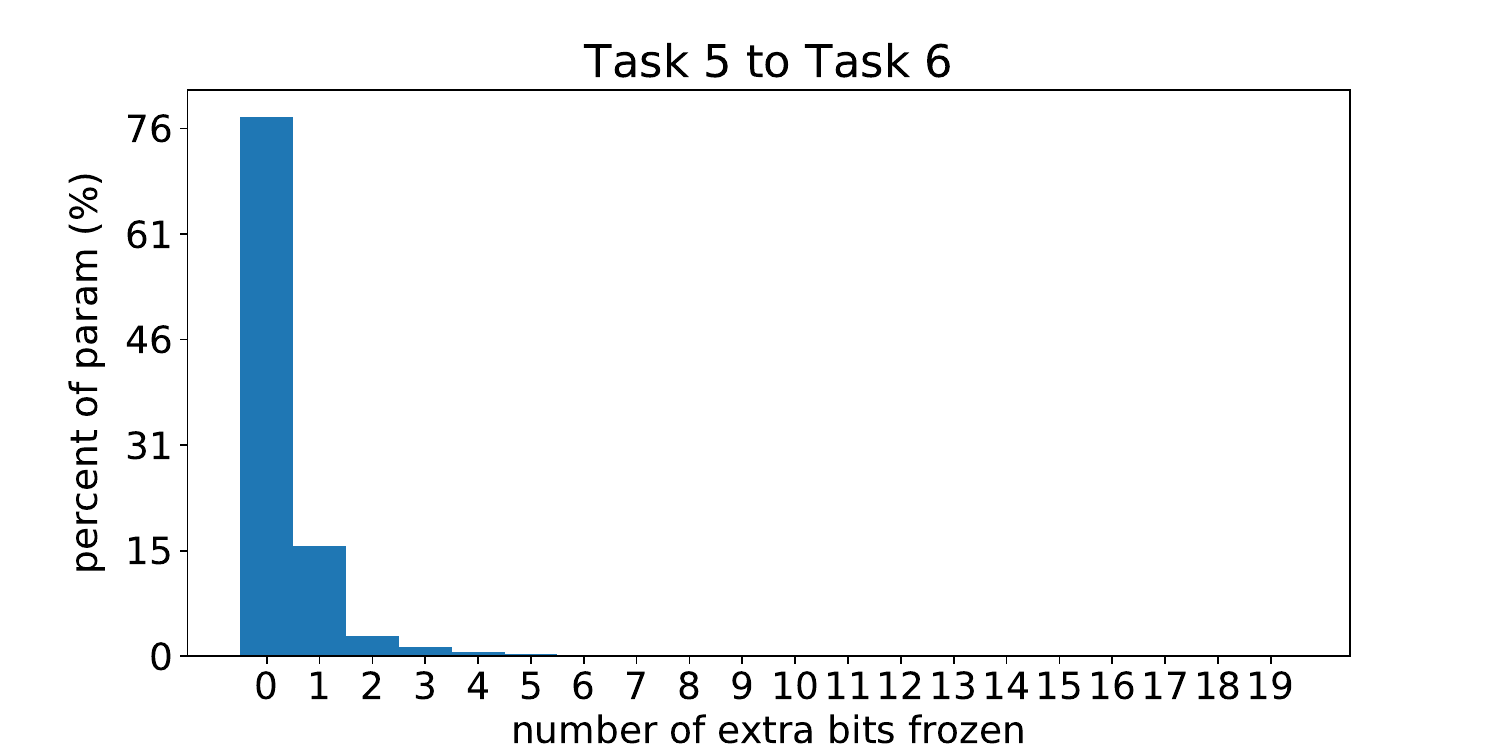}
    \caption{\textbf{Bit freezing visualization in histogram} Each histogram in this figure shows what percentage of parameters have how many extra bits frozen after learning one task.}
    \label{bit_freezing_hist}
\end{figure}

\end{document}